\documentclass{article}
\usepackage{ijcai22}

\usepackage{times}
\usepackage{soul}
\usepackage{url}
\usepackage[hidelinks]{hyperref}
\usepackage[utf8]{inputenc}
\usepackage[small]{caption}
\usepackage{graphicx}
\usepackage{amsmath}
\usepackage{amsthm}
\usepackage{booktabs}
\usepackage{algorithm}
\usepackage{algorithmic}
\usepackage{amssymb}
\usepackage{multirow}
\usepackage{subfigure}

\newtheorem{theorem}{Theorem}
\newtheorem{definition}{Definition} 

\title{Auction-Based Ex-Post-Payment Incentive Mechanism Design for Horizontal Federated Learning with Reputation and Contribution Measurement}

\author{
Jingwen Zhang\and
Yuezhou Wu\And
Rong Pan
}

\begin{document}

\maketitle

\begin{abstract}
Federated learning trains models across devices with distributed data, while protecting the privacy and obtaining a model similar to that of centralized ML. A large number of workers with data and computing power are the foundation of federated learning. However, the inevitable costs prevent self-interested workers from serving for free. Moreover, due to data isolation, task publishers lack effective methods to select, evaluate and pay reliable workers with high-quality data. Therefore, we design an auction-based incentive mechanism for horizontal federated learning with reputation and contribution measurement. By designing a reasonable method of measuring contribution, we establish the reputation of workers, which is easy to decline and difficult to improve. Through reverse auctions, workers bid for tasks, and the task publisher selects workers by combining reputation and bid price. With the budget constraint, winning workers are paid based on performance. We prove that our mechanism satisfies the individual rationality of the honest workers, budget feasibility, truthfulness, and computational efficiency. The experimental results show its effectiveness compared to the benchmarks. 
\end{abstract}

\section{Introduction}
Federated learning (FL) is a novel machine learning framework that uses distributed data and trains models across devices \cite{yang2019federated}. Everyone keeps the data locally, and trains a global model collaboratively without sharing data but only sharing the model's parameters. FL effectively breaks down data silos, makes full use of everyone's storage and computing capabilities, and satisfies privacy protection, data security, and government laws. FL is used in banking \cite{yang2019ffd}, healthcare \cite{dayan2021federated}, transportation \cite{liu2020privacy}, etc. Researches on FL mainly focus on privacy protection \cite{so2020byzantine}, heterogeneity \cite{wang2020optimizing}, communication efficiency \cite{chen2021communication}, etc., where workers voluntarily participate in FL. However, due to the inevitable costs, self-interested workers will not serve for free. Moreover, due to data isolation, the publisher cannot know the worker's data quality and reliability. It is impractical to require workers to report these truthfully. Therefore, task publishers lack the means to select as many high-quality workers as possible. In addition, even if workers are selected and rewarded, allocating rewards fairly remains a great challenge. This requires designing an effective method for measuring the contribution, quantifying the performance and reflecting the nature of workers.

We designed a reverse auction-based ex-post-payment incentive mechanism for horizontal federated learning with reputation and contribution measurement to motivate more workers and help publishers to select, evaluate, and pay workers rationally. By designing a reasonable contribution measurement method, we establish the reputation of workers. Reputation can indirectly reflect the quality and reliability of workers. Through the reverse auction, workers submit bid prices. Because of their pricing power, workers are motivated. The publisher selects workers with a high reputation and low bid price and pays them according to their performance. The contributions of the paper can be summarized as follows.
\begin{itemize}
    \item We propose a method of contribution measurement. 
    \item We have established a reputation system. The reputation is easy to decline and difficult to improve.
    \item Combining reputation and auction, we propose a mechanism to select workers and pay based on performance.
    \item We prove that our mechanism satisfies individual rationality of the honest workers, budget feasibility, truthfulness, and computational efficiency. The experimental resultse demonstrate its effectiveness.
\end{itemize}

The rest of the paper is organized as follows. Section~\ref{sec:related_work} introduces related work. Section~\ref{sec:system_model_problem_def} describes the system model and problem definition. Section~\ref{sec:mechanism_design} designs the mechanism in detail. Section~\ref{sec:theoretical_analysis} conducts theoretical analysis and Section~\ref{sec:experiment} shows the simulation results.

\section{Related Work}
\label{sec:related_work}
Shapley Value (SV) is used to measure the contribution of workers \cite{wei2020efficient}. SV is approximated by group testing~\cite{liu2021gtg} or sampling~\cite{wang2020principled}, which still takes a long time. The local model accuracy is indirectly or directly as the contribution~\cite{nishio2020estimation}, but it cannot express their mutual influence. The similarity or distance of model parameters is used for contribution measurement \cite{xu2021reputation}, which may make the reputation of high-quality workers low. Zhang \textit{et al.}~\shortcite{zhang2021blockchain} leverages contributions to update workers’ reputations. Kang \textit{et al.}~\shortcite{kang2019incentive} proposes a subjective logic model, combining local and recommended opinions to calculate reputation. Zhao \textit{et al.}~\shortcite{zhao2020mobile} sets an initial reputation. When submitting a useful model, reputation increases by 1, otherwise, decreases by 1. However, the granularity of reputation is large and lacks discrimination. Le \textit{et al.}~\shortcite{le2021incentive} uses reverse auction to help task publishers select workers to maximize social welfare. Zeng \textit{et al.}~\shortcite{zeng2020fmore} proposes a multi-directional auction mechanism that takes resource differences into account. Le \textit{et al.}~\shortcite{le2020auction} adoptes a random reverse auction mechanism to minimize social costs. Sarikaya and Ercetin~\shortcite{sarikaya2019motivating} model the interaction between workers and the publisher as a Stackelberg game to motivate and coordinate each worker. Kang \textit{et al.}~\shortcite{kang2019incentive} uses contract theory to design an incentive mechanism. Workers with different data quality choose a contract item to maximize their utilities. These papers determine workers' rewards before the task and do not consider that the workers might not work according to the claimed plan.

\section{System Model and Problem Definition}
\label{sec:system_model_problem_def}
\subsection{System Model}
A federated learning system consists of a task publisher and lots of workers. Workers are users or companies of smart devices with data. They are independent and non-colluding, wanting to participate in FL tasks to obtain profits. Their accumulated reputation $Re_i$ is public, but the data quality, quantity, and task cost $c_i$ are private. The task publisher with a budget of $B$ releases a task to recruit workers. His data serves as the test set and validation set. The interaction between the publisher and the workers is modeled as a reverse auction. The worker $i$ submits the sealed bid price $b_i$. Since workers are rational, and follow the strategy to maximize their utilities, the bid price $b_i$ may not be equal to the true cost $c_i$. The publisher obtains the accumulated reputation of each worker and selects the worker combining the bid price. Then the FL task starts. In each global round, the publisher checks the quality of each local model and measures each worker's contribution. Finally, the publisher calculates the internal reputation of the winning worker $i$, updates his accumulated reputation, and pays $p_i$. Figure~\ref{fig:system_model} describes the detailed process.
\begin{figure}[tb]
  \centering
  \centerline{\includegraphics[width=\linewidth]{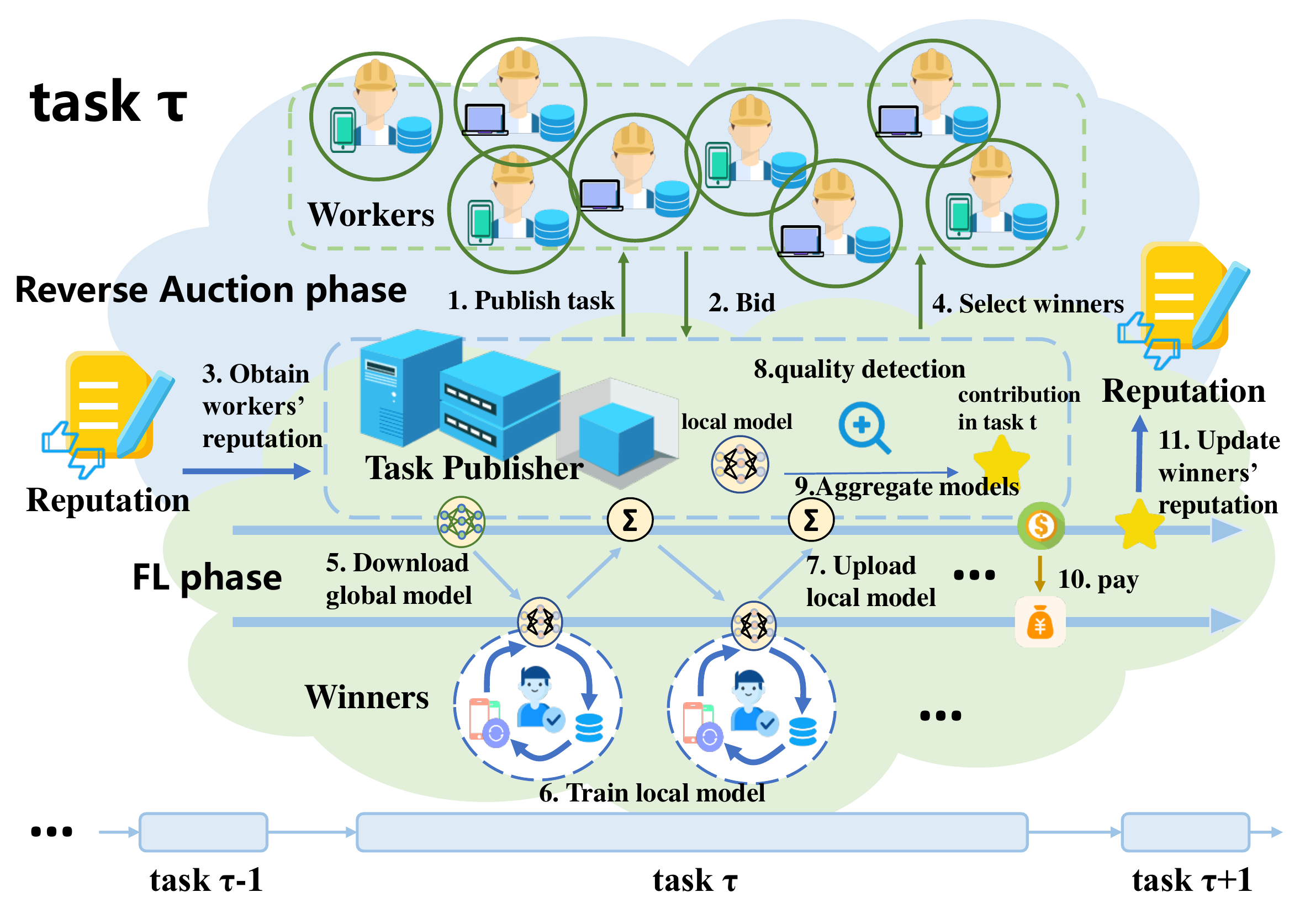}}
  \caption{System Model}
  \label{fig:system_model}
\end{figure}

\subsection{Problem Definition}
To select more high-quality workers to obtain a higher-quality model with a limited budget, it is necessary to design an incentive mechanism $\mathbb{M}(\mathbf{f},\mathbf{p}) $, including a selection mechanism $\mathbf{f}$ and a payment mechanism $\mathbf{p}$. Suppose the selected workers form the set $S$. The utility $u_i$ of the worker $i$ is
\begin{equation}
    u_i = 
    \left\{
        \begin{aligned}
        0 & , & i \notin S, \\
        p_i - c_i & , & i \in S.
        \end{aligned}
    \right.
\end{equation}
Since the budget is limited, $\sum_{i \in S}p_i \leq B$ need to be satisfied. The publisher wants to select as many high-quality workers as possible. Since the accumulated reputation indirectly reflects quality, we model the publisher’s utility as $\mathcal{U} = \sum_{i \in S} Re_i$. The goal of our mechanism is to maximize $\mathcal{U}$ by determining the winning worker set $S$ while satisfying the following four economic properties. \textit{Individual Rationality of the Honest Workers}: Honest workers have non-negative utility. \textit{Budget Feasibility}: The total payment cannot exceed the budget. \textit{Computational Efficiency}: The time complexity of the mechanism is polynomial. \textit{Truthfulness}: Reporting true cost can maximize the utility of the worker.

\section{Mechanism Design} 
\label{sec:mechanism_design}
\subsection{Contribution Measurement}
An intuitive method to measure the contribution is to evaluate the performance of the local model on the validation set, such as accuracy and loss \cite{lyu2020collaborative}. However, it is not fair, because the prediction difficulty of different validation samples is different. The prediction difficulty of a certain validation sample is low when many workers predict it correctly. Even if one of the workers is removed, the global model is still very likely to predict it correctly. The value of this sample being correctly predicted is relatively small. On the contrary, when a sample is difficult to predict, the workers who can correctly predict it are more important in this sample. The value of this sample being correctly predicted is relatively large. Therefore, the sample that is more difficult to predict is given a larger weight, and the contribution of workers is according to their performance on the weighted samples. Directly leveraging accuracy as worker contribution is equivalent to assigning equal weight to each validation sample, which may lack fairness. Our contribution measurement is divided into two steps. First, determine the weight of each sample, and then evaluate the worker's performance on the weighted sample. The weight of the validation sample is determined by the predictive ability of all workers, and the predictive ability is determined by the probability that the predicted result is the true label. The probability that the worker $i$ in the worker set $U$ correctly predicts the sample $j$ in the validation set $D$ is $P_{i,j}$, then the weight of the sample $j$ is 
\begin{equation}
    w_j = \frac{\sum_{i \in U}-\ln P_{i,j}}{\sum_{i\in U}\sum_{j \in D}- \ln P_{i,j}}.
\end{equation}
Taking $P_{i,j}$ as the performance of worker $i$ on sample $j$, the contribution of worker $i$ is computed by
\begin{equation}
    contrib_i = \sum_{j \in D} P_{i,j}w_j.
\end{equation}
The above is the contribution of worker $i$ in a global round. Denote the contribution of the worker $i$ in the round $t$ as $contrib_i^t$, and standardized it by $contrib_i^t = \frac{contrib_i^t}{\max_{k \in U}(contrib_k^t)}$. The contribution $contrib_i$ of worker $i$ to the task is calculated by $contrib_i = \frac{\sum_{t \in \mathcal{T}} contrib_i^t}{|\mathcal{T}|}$, where $\mathcal{T}$ is all global rounds that worker $i$ participated in.

\subsection{Reputation Modeling}
Reputation is a rating of the quality and reliability of workers, allowing publishers to select high-quality workers. First, we model the reputation of a worker in a certain task  as internal reputation and then integrate the reputation in all historical tasks as the accumulated reputation. The internal reputation represents the performance of the worker in the current task, which is related to his contribution and the quality detection of his local model. We use the method proposed by Zhang~\shortcite{zhang2021incentive} to check the quality of the local model. The loss of the global model in the validation set when the worker $i$ participates and does not participate in the model aggregation is $l$ and $l_{-i}$ respectively. Set a predefined threshold $\delta = -0.005$, if $\Delta l_i = l_{-i}-l \geq \delta$, worker $i$ passes the detection, and only the passed local model can participate in the model aggregation. To model the internal reputation, we introduce the trustworthiness $trust_i$ of worker $i$, which indicates the degree of acceptance of his local model. $trust_i$ is the output of the Gompertz function, and the input is shown in Eq.~(\ref{equ:trust_degree}), where $n_i^{pass}$ and $n_i^{fail}$ are the number of times that the worker $i$ passed and failed the detection, respectively, and $\theta \in (0, 0.5)$ is a predefined parameter, meaning that more attention is paid to failing the detection. We set $\theta = 0.4$.
\begin{equation}
    \label{equ:trust_degree}
    x_i = \frac{\theta \cdot n_i^{pass} - (1-\theta)\cdot n_i^{fail}}{\theta \cdot n_i^{pass} + (1 - \theta) \cdot n_i^{fail}}.
\end{equation}
The Gompertz function is a growth curve, suitable for modeling the trust of individual interactions \cite{zhang2021incentive}. The Gompertz function is described as $y = a\exp(b\exp(c\cdot x))$, where $a$, $b$, and $c$ are parameters, $x$ is input, and $y$ is output. We set $a = 1$, $b = -1$, $c = -5.5$. The trustworthiness of worker $i$ is $trust_i = \exp(-\exp(-5.5x_i))$. Then, combine the $contrib_i \in [0, 1]$ and $trust_i \in [0, 1]$ of the worker $i$ in the task $\tau$ to obtain his internal reputation $re_i^\tau \in [0, 1]$. Omit the superscript $\tau$ without confusion, as
\begin{equation}
    re_i = contrib_i \cdot trust_i
\end{equation}

The accumulated reputation $Re_i^\tau$ of the worker $i$ is derived from the internal reputation of all historical tasks. The newer the internal reputation, the more it reflects the nature of the worker. Thus using the moving average method to model the accumulated reputation of workers $i$, as shown in the Eq.~(\ref{equ:accumulated_reputation}), where $\alpha$ is the attenuation coefficient and the weight of the latest internal reputation.
\begin{equation}
    \label{equ:accumulated_reputation}
    Re_i^\tau = \alpha \cdot re_i^\tau + (1 - \alpha) \cdot Re_i^{\tau - 1}.
\end{equation}
When workers perform good tasks continuously, their accumulated reputation should gradually increase slightly. Once a bad task is done, their accumulated reputation should immediately drop significantly. Therefore, we consider the number of consecutive good and bad tasks and the latest internal reputation as factors to dynamically determine the attenuation coefficient $\alpha$. The related concepts are defined below. 
\begin{definition}[Honest worker, Dishonest worker]
If a winning worker's internal reputation in the current task is not less than his accumulated reputation beforehand, then he is called an honest worker, otherwise, a dishonest worker.
\end{definition}
\begin{definition}[Number of consecutive good tasks]
From the initial task or the last bad task to the present, the number of tasks where worker $i$ continuously performs honestly is called the number of consecutive good tasks $n_i^+$, which satisfies
\begin{equation}
    n_i^+ = 
    \left\{
        \begin{aligned}
        0 & , & re_i^\tau < Re_i^{\tau-1}, \\
        n_i^+ + 1 & , & re_i^\tau \geq Re_i^{\tau-1}.
        \end{aligned}
\right.
\end{equation}
\end{definition}
\begin{definition}[Number of consecutive bad tasks]
From the initial task or the last good task to the present, the number of tasks where worker $i$  continuously performs dishonestly is called the number of consecutive bad tasks $n_i^-$, which is
\begin{equation}
    n_i^- = 
    \left\{
        \begin{aligned}
        0 & , & re_i^\tau \geq Re_i^{\tau-1}, \\
        n_i^- + 1 & , & re_i^\tau < Re_i^{\tau-1}.
        \end{aligned}
    \right.
\end{equation}
\end{definition}

Honest performance will slightly increase accumulated reputation, while dishonest performance will greatly reduce it. That is, the smaller $re_i^\tau$, the greater $\alpha$. We introduce $h(re^\tau)$ to represent the influence of $re^\tau$ on $\alpha$. $h(re^\tau)$ is the decreasing function of $re^\tau \in [0, 1]$ and ensure that when $Re_i^{\tau-1} = Re_j^{\tau-1}$ and ignore the number of consecutive good and bad tasks, if $re_i^\tau> re_j^\tau$ , then $Re_i^\tau> Re_j^\tau$. To determine the form of $h(re^\tau)$, we make the difference between $Re_i^\tau$ and $Re_j^\tau$, as shown in Eq.~(\ref{equ:reputation_diff}). Since $Re_i^{\tau-1} = Re_j^{\tau-1}$, the two are abbreviated as $Re^{\tau-1}$.
\begin{equation}
    \label{equ:reputation_diff}
    \resizebox{1.0\linewidth}{!}{$
    \displaystyle
    Re_i^\tau - Re_j^\tau = (h(re_i^\tau)re_i^\tau - h(re_j^\tau)re_j^\tau) - (h(re_i^\tau) - h(re_j^\tau))Re^\tau.
    $}
\end{equation}
To satisfy $Re_i^\tau> Re_j^\tau$, only $h(re_i^\tau)re_i^\tau-h(re_j^\tau)re_j^\tau> 0$ is required, meaning that $q (re^\tau) = h(re^\tau)re^\tau$ is an increasing function of $re^\tau$. Based on the above analysis, $h(re^\tau)$ is
\begin{equation}
    h(re^\tau) = - \frac{19}{10\pi} \arctan(\frac{10re^\tau}{\pi}) + 1.
\end{equation}
With the same other factors, the more the number of consecutive good tasks, the higher the accumulated reputation. That is, when $re_i^\tau \geq Re_i^{\tau-1}$, the larger $n_i^+$, the smaller the weight of $Re_i^{\tau-1}$. We introduce $g(n^+)$ to represent the influence of $n_i^+$ on $(1-\alpha)$. $g(n^+)$ is the decreasing function of $n^+$. With the same other factors, the more the number of consecutive bad tasks, the lower the accumulated reputation. That is, when $re_i^\tau < Re_i^{\tau-1}$, the larger $n_i^-$, the smaller the weight of $Re_i^{\tau-1}$. For simplicity, we also leverage $g(n^-)$ to represent the influence of $n^-$ on $(1-\alpha)$. $g(n)$ satisfies $g(n)> 0$, and has its upper and lower bounds. When $n^+ = 0$ or $n^- = 0$, it does not affect $(1-\alpha)$, so $g(0) = 1$ must be satisfied. And when $n^+$ or $n^-$ grows, $g(n)$ tends to a certain value but not 0. We set $g(n)$ to be
\begin{equation}
    g(n) = \frac{2(1 - \beta_2)}{1 + \exp(\beta_1 n)} + \beta_2,
\end{equation}
where $\beta_1$ and $\beta_2$ are parameters. The larger the $\beta_1$, the faster the decline of $g(n)$. $\beta_2$ controls the asymptotes of $g(n)$. We set $\beta_1 = \frac{1}{2}$, $\beta_2 = \frac{1}{4}$. Combine the effects of $n^+$ and $n^-$ on $(1-\alpha)$ as $f(n^+, n^-) = g(n^+)g(n^-)$. Eqs.~(\ref{equ:adjust_a}) and (\ref{equ:final_re}) model the accumulated reputation of workers $i$.
\begin{align}
    \label{equ:adjust_a}
    \alpha_i = \frac{h(re_i^\tau)}{h(re_i^\tau) + (1 - h(re_i^\tau))f(n_i^+, n_i^-)}. \\
    \label{equ:final_re}
    Re_i^\tau = \alpha_i \cdot re_i^\tau + (1 - \alpha_i) \cdot Re_i^{\tau - 1}.
\end{align}

\texttt{FedAvg} \cite{mcmahan2016federated} is a method of aggregating models according to the amount of data. However, workers are likely to lie about the amount, which leads to a poor global model. If local models are aggregated by direct averaging, better local models will not receive enough attention, which slows down global convergence. Therefore, when aggregating, we assign different weights to local models according to performance. Define the contribution score of worker $i$ as $contrib\_score_i^t = \max(0, contrib_i^t)$. Define the quality detection score of worker $i$ as $quality\_score_i^t = \frac{s_0 + s_i}{\sum_{k}(s_0 + s_k)}$, where $s_0$ is a constant, usually $s_0 = 1$, and $s_i$ is defined as $s_i = \frac{\Delta l_i - min_j(\Delta l_j)}{\sum_{k}(\Delta l_k - min_j(\Delta l_j))}$. Finally, the aggregate weight of the worker $i$ is $\omega_i = \frac{contrib\_score_i^t \cdot quality\_score_i^t}{\sum_j contrib\_score_j^t \cdot quality\_score_i^t}$.

\subsection{Workers Selection and Payment}
Existing auction-based researches determines the final reward of workers before the task. Since reward has nothing to do with actual performance, workers may not work according to the claimed plan, which will affect the global model.To solve the above challenges, we designed an ex-post payment incentive mechanism $\mathbb{M}(\mathbf{f}, \mathbf{p})$ based on reputation and proportional share reverse auction \cite{singer2010budget}. Our selection mechanism $\mathbf{f}$ is consistent with the proportional share mechanism, while our payment mechanism $\mathbf{p}$ is improved based on it so that the reward is determined according to the worker’s performance. That is the so-called \textit{ex-post payment}. Algorithm~\ref{alg:offline_auction} describes the process of the mechanism.

\begin{algorithm}[tb]
    \caption{Workers Selection and Payment}
    \label{alg:offline_auction}
    \textbf{Input}: Budget $B$; Worker Set $U$; 
    \begin{algorithmic}[1]  
        \label{alg:offline_auction:sort_and_choose:start}
        \STATE Sort the workers in $U$ so that $\frac{b_1}{Re_1} \leq ... \leq \frac{b_{|U|}}{Re_{|U|}}$;
        \label{alg:offline_auction:sort}
        \STATE $i = 1$; $S = \phi$;
        \WHILE {$\frac{b_i}{Re_i} \leq \frac{B}{Re_i + \sum_{j \in S} Re_j} $}
        \label{alg:offline_auction:choose_worker:start}
            \STATE $S = S \cup \{i\}$; $i = i + 1$;
        \ENDWHILE
        \label{alg:offline_auction:choose_worker:end}
        \label{alg:offline_auction:sort_and_choose:end}
        \FOR {each worker $i \in U\setminus S$}
        \label{alg:offline_auction:upper_payment:start}
            \STATE $p_i = 0$;
        \ENDFOR
        \STATE $k = i - 1$; $\rho^{\ast} = \min(\frac{b_{k+1}}{Re_{k+1}}, \frac{B}{\sum_{j \in S}Re_j})$;
        \FOR {each worker $i \in S$}
            \STATE $p_{i}^{up} = Re_i \cdot \rho^{\ast} $; 
        \ENDFOR
        \label{alg:offline_auction:upper_payment:end}
        \STATE \textit{/*After the task:*/}
        \FOR {each worker $i \in S$}
        \label{alg:offline_auction:final_payment:start}
            \STATE $p_i^{\prime} = \max(\frac{B\cdot re_i}{\sum_{j \in S} re_j}, \rho^{\ast}\cdot re_i)$; $p_{i} = \min(p_i^{up}, p_i^{\prime})$; 
        \ENDFOR
        \label{alg:offline_auction:final_payment:end}
    \end{algorithmic}
\end{algorithm}

The publisher needs to select more high-quality workers to get a high-precision model. To select more workers, the publisher tend to select workers with lower bid prices. To select higher-quality workers, he tends to select workers with higher accumulated reputations. Balancing the bid price and accumulated reputation, we define the worker’s unit accumulated reputation bid price as his cost density $\rho_i = \frac{b_i}{Re_i}$. We sort all workers in non-descending order of their cost density. According to the proportional share allocation rule, we find the last worker $k$ in the sequence that satisfies $\rho_k \leq B / (Re_k + \sum_{i = 1}^{k-1}Re_i)$ from front to back. The first $k$ workers in the sequence form the winning worker set $S$. To determine the reward, we define the payment density threshold $\rho^{\ast}$ as
\begin{equation}
    \rho^{\ast} = \min(\frac{B}{\sum_{i \in S}Re_i}, \frac{b_{k+1}}{Re_{k+1}}).
\end{equation}
The losing worker is paid 0. The winning worker has a reward upper bound which ensure truthfulness and budget feasibility. The reward upper bound of the winning worker $i \in S$ is
\begin{equation}
    p_i^{up} = Re_i \cdot \rho^{\ast}.
\end{equation}
After the task, the internal reputation $re_i$ of each winning worker $i$ is evaluated and his temporary reward is $p_i^{\prime} = re_i \cdot \max(\frac{B}{\sum_{j \in S}re_j}, \rho^{\ast})$. His final reward is
\begin{equation}
    p_i = \min(p_i^{up}, p_i^{\prime}).
\end{equation}

\section{Theoretical Analysis}
\label{sec:theoretical_analysis}
\begin{theorem}
Our mechanism satisfies the individual rationality of the honest workers.
\end{theorem}
\begin{proof}
Worker $i$ is honest, meaning that he wins and $re_i \geq Re_i$. $p_i^{\prime} = \max(\rho^{\ast}re_i, \frac{B\cdot re_i}{\sum_{j \in S}Re_j}) \geq \rho^{\ast }re_i \geq \rho^{\ast}Re_i = p_i^{up}$. Thus, $p_i = \min(p_i^{\prime}, p_i^{up}) = p_i^{up}$. Worker $i$ wins, so $\frac{b_i}{Re_i} \leq \frac{b_{k+1}}{Re_{k+1}}$ and $\frac{b_i}{Re_i} \leq \frac{B}{\sum_{j \in S}Re_j}$. Thus $b_i \leq \min(\frac{b_{k+1}\cdot Re_i}{Re_{k+1}}, \frac{B\cdot Re_i}{\sum_{j \in S}Re_j}) = p_i^{up} = p_i$.
\end{proof}

\begin{theorem}
Our mechanism satisfies budget feasibility.
\end{theorem}
\begin{proof}
Total reward is $\sum_{i \in S}p_i \leq \sum_{i \in S} p_i^{up} =  \sum_{i \in S}\min(\frac{b_{k+1}\cdot Re_i}{Re_ {k+1}}, \frac{B\cdot Re_i}{\sum_{j \in S}Re_j}) \leq \frac{\sum_{i \in S} B \cdot Re_i}{\sum_{j \in S}Re_j} = B$.
\end{proof}

\begin{theorem}
Our mechanism satisfies computational efficiency.
\end{theorem}
\begin{proof}
In Algorithm~\ref{alg:offline_auction}, the time complexity of sorting the workers in $U$ (line \ref{alg:offline_auction:sort}) is $O(n\log_2 |U|)$, that of selecting workers (line \ref{alg:offline_auction:choose_worker:start}-\ref{alg:offline_auction:choose_worker:end}) is $O(|S|)$, that of computing the reward of the losing worker and the reward upper bound of the winning worker (line \ref{alg:offline_auction:upper_payment:start}-\ref{alg:offline_auction:upper_payment:end}) is $O(|U|)$ and that of calculating the final reward (line \ref{alg:offline_auction:final_payment:start}-\ref{alg:offline_auction:final_payment:end}) is $O(|S|)$. Therefore, the time complexity of Algorithm~\ref{alg:offline_auction} is $O(|U| \log_2 |U|)$.
\end{proof}

\begin{theorem}
Our mechanism satisfies truthfulness.
\end{theorem}
\begin{proof}
Assume that regardless of the bid price, $re_i$ is the same. Consider the scenario where $b_i>c_i$ first. \\
\textbf{Case 1: worker $i$ wins with both $b_i$ and $c_i$}. The winning workers set remains unchanged, resulting in the same payment density. Since $re_i$ is the same, $u(c_i, b_{-i}) = u(b_i, b_{-i})$. \\
\textbf{Case 2: worker $i$ wins with $c_i$ but loses with $b_i$}. Submitting $b_i$ leads to $u(b_i, b_{-i}) = 0$. As long as $re_i \geq Re_i$, submitting $c_i$ leads to $u(c_i, b_{-i}) = p_i-c_i = \min(\max(\frac{B \cdot re_i}{\sum_{j \in S}re_j}, \rho^{\ast} \cdot re_i), \rho^{\ast} \cdot Re_i)-c_i = \rho^{\ast} \cdot Re_i-c_i \geq 0 = u(b_i, b_{-i})$. \\
\textbf{Case 3: worker $i$ wins with $b_i$ but loses with $c_i$}. This means $\frac{b_i}{Re_i} \leq \frac{c_i}{Re_i}$, which contradicts $b_i> c_i$. Thus it is impossible. \\
\textbf{Case 4: worker $i$ loses with both $c_i$ and $b_i$}. This means $ u (c_i, b_ {-i}) = u (b_i, b_ {-i}) = 0 $. 

Next, discuss the scenario where $b_i <c_i$. \\
\textbf{Case 1: worker $i$ wins with both $b_i$ and $c_i$}. The winning workers set remains unchanged, resulting in the same payment density. Since $re_i$ is the same, $u(c_i, b_{-i}) = u(b_i, b_{-i})$. \\
\textbf{Case 2: worker $i$ wins with $c_i$ but loses with $b_i$}. This means $\frac{c_i}{Re_i} \leq \frac{b_i}{Re_i}$, which contradicts $b_i <c_i$. Thus it is impossible.\\
\textbf{Case 3: worker $i$ wins with $b_i$ but loses with $c_i$}. When submitting $c_i$, $k$ workers win, not including worker $i$, so $u(c_i, b_{-i}) = 0$ and $\frac{b_{k+1}}{ Re_{k+1}} \leq \frac{c_i}{Re_i}$. When submitting $b_i$, $k^{\prime}$ workers win, including worker $i$, so $\frac{b_i}{Re_i} \leq \frac{B}{Re_i + \sum_{j \in S}Re_j}$ and $\frac{b_i}{Re_i} \leq \frac{b_{k^{\prime}+1} }{Re_{k^{\prime}+1}}$. Compared with $c_i$, when submitting $b_i$, $k^{\prime} \leq k$. Thus $\frac{b_i}{Re_i} \leq \frac{b_{k^{\prime}+1}}{Re_{k^{\prime}+1}} \leq \frac{b_{k+1 }}{Re_{k+1}} \leq \frac{c_i}{Re_i}$. When submitting $b_i$, utility of worker $i$ is $u(b_i, b_{-i}) = p_i-c_i \leq \rho^{*} \cdot Re_i-c_i \leq \frac{b_{k^{ \prime}+1}}{Re_{k^{\prime}+1}} Re_i-c_i \leq \frac{c_i}{Re_i} Re_i-c_i = 0 = u(c_i, b_{-i})$.\\
\textbf{Case 4: worker $i$ loses with both $c_i$ and $b_i$}. This means $ u (c_i, b_ {-i}) = u (b_i, b_ {-i}) = 0$.
\end{proof}

\section{Experiment}
\label{sec:experiment}
We conduct experiments using the MNIST dataset with a fully connected model and the FashionMNIST dataset with the LeNet model. The publisher has a validation set and a test set, both of size 5000. Each worker has a training set of size 1000, which is randomly sampled from the corresponding dataset. By modifying the label to another, the data accuracy of each worker may be different. The winning worker will train the local model for 1 epoch with a learning rate of 0.05 and a batch size of 128.
\begin{figure}[tb]
  \centering
  \centerline{\includegraphics[width=\linewidth]{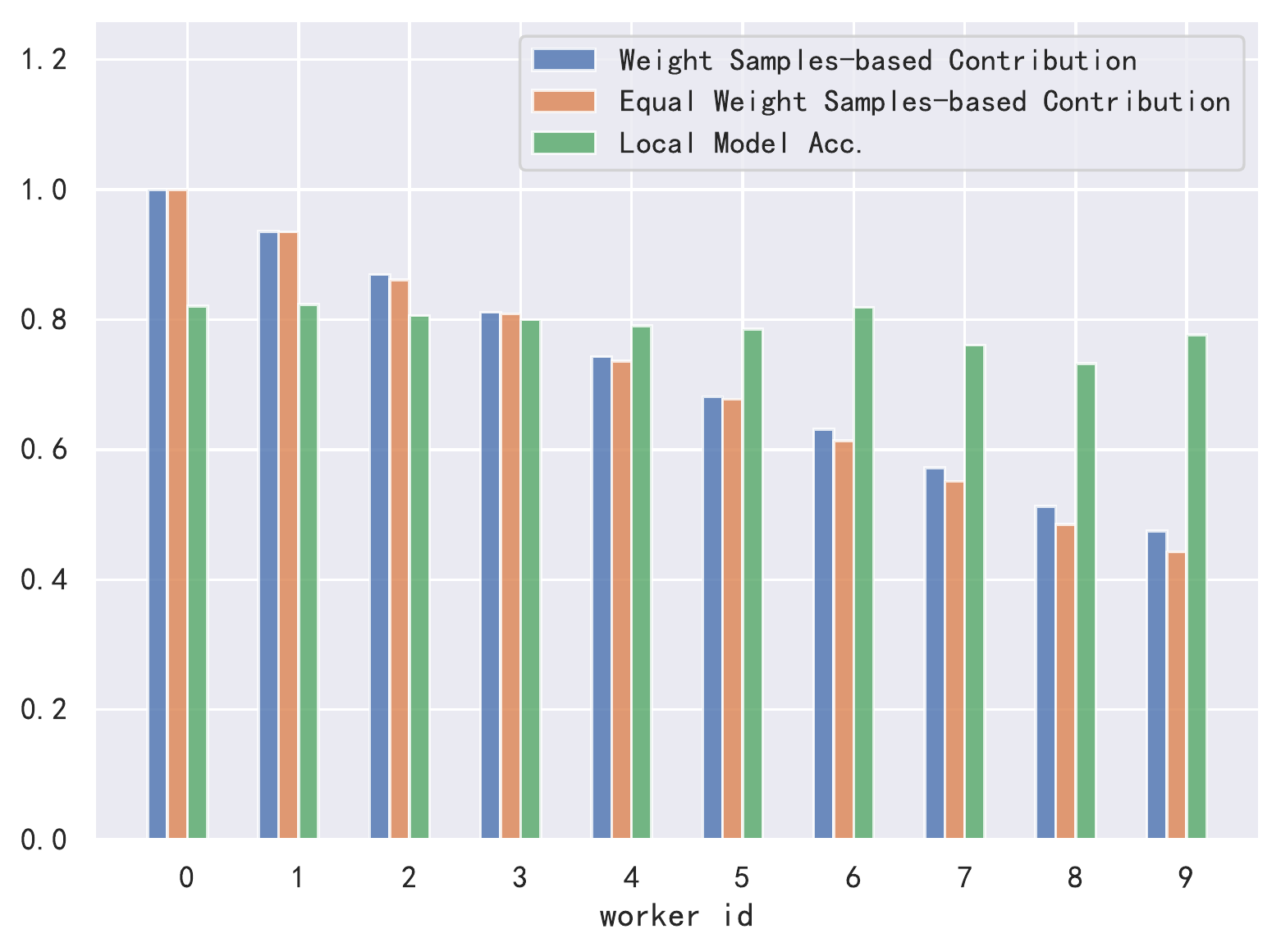}}
  \caption{Contribution measurement results in Case 1}
  \label{fig:contribution_case1}
\end{figure}

\begin{figure}[tb]
  \centering
  \centerline{\includegraphics[width=\linewidth]{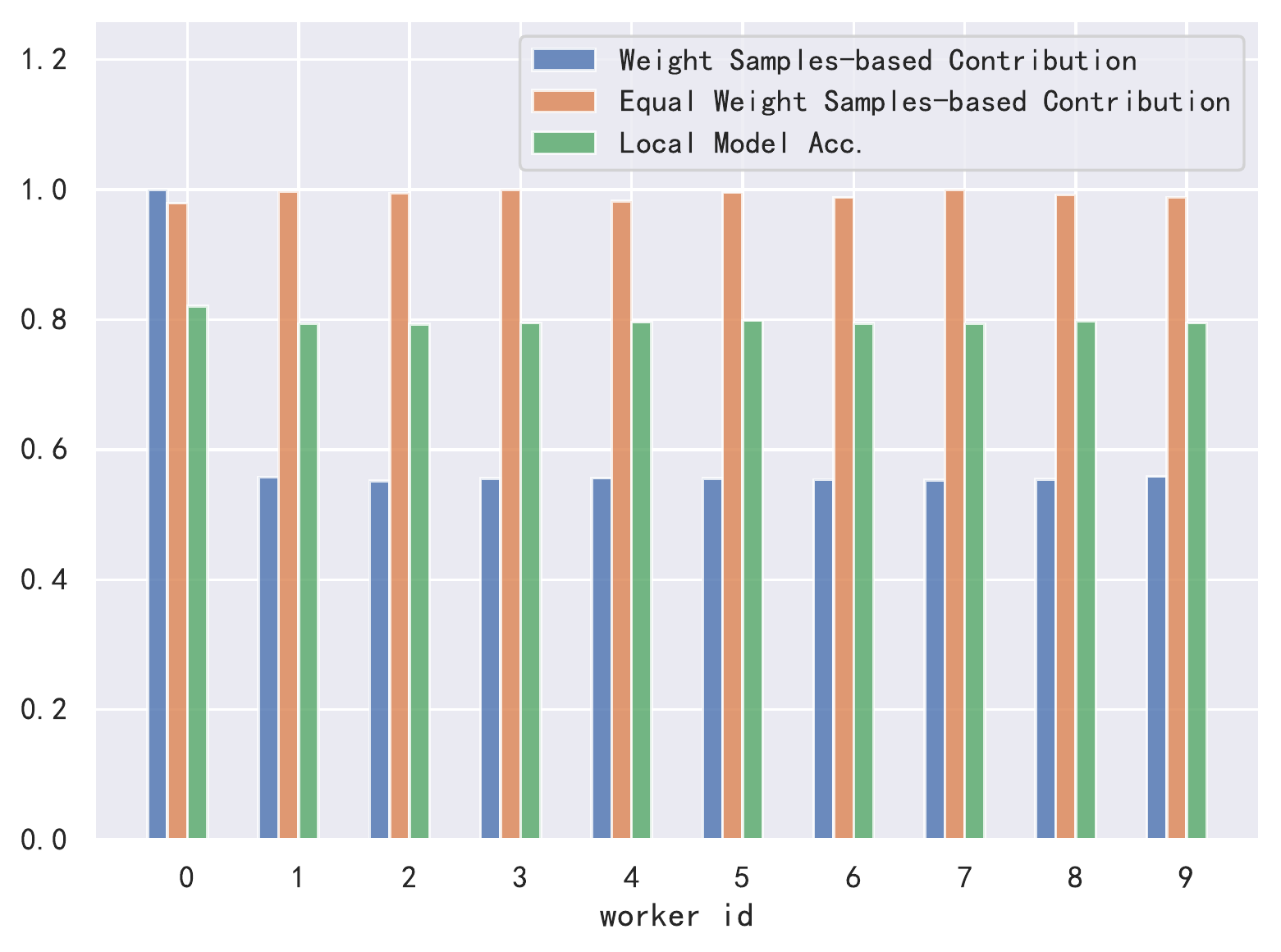}}
  \caption{Contribution measurement results in Case 2}
  \label{fig:contribution_case2}
\end{figure}
We use the equal-weighted sample-based method as a benchmark which is the same as our weighted sample-based method, but the weights of the samples are equal. We compare our contribution measure with this benchmark using the MNIST dataset in two cases. The first case is that the data of 10 workers is iid, whose accuracy decreases sequentially by  0.1 from 1.0 to 0.1. The second case is that the data of 10 workers is non-iid and the accuracy is 1.0, in which worker 0 has data of all labels and the others lack data of a certain label. Figure~\ref{fig:contribution_case1} shows the results for Case 1, where the contribution decreases as the data accuracy decreases in both ours and the benchmark. However, the accuracy of the local model is similar, which cannot reflect the data quality. Figure~\ref{fig:contribution_case2} shows the results for Case 2, where our method can highlight the contribution of higher-quality worker 0 compared to the benchmark and local model accuracy. These results demonstrate that our method can effectively and fairly measure contribution.

We compare our mechanism with the following benchmarks. The first is RRAFL. The second is Vanilla FL, where workers are randomly selected until the budget is exhausted. The third is the Proportional Share mechanism. The fourth is Bid Greedy, where workers with low bids are preferentially selected until the budget is exhausted. The fifth is Reputation Greedy, preferring high-reputation workers until the budget runs out. The sixth is an Approximate Optimal approach which has full knowledge about all workers’ information. 
\begin{figure}[tb]
  \centering
  \centerline{\includegraphics[width=\linewidth]{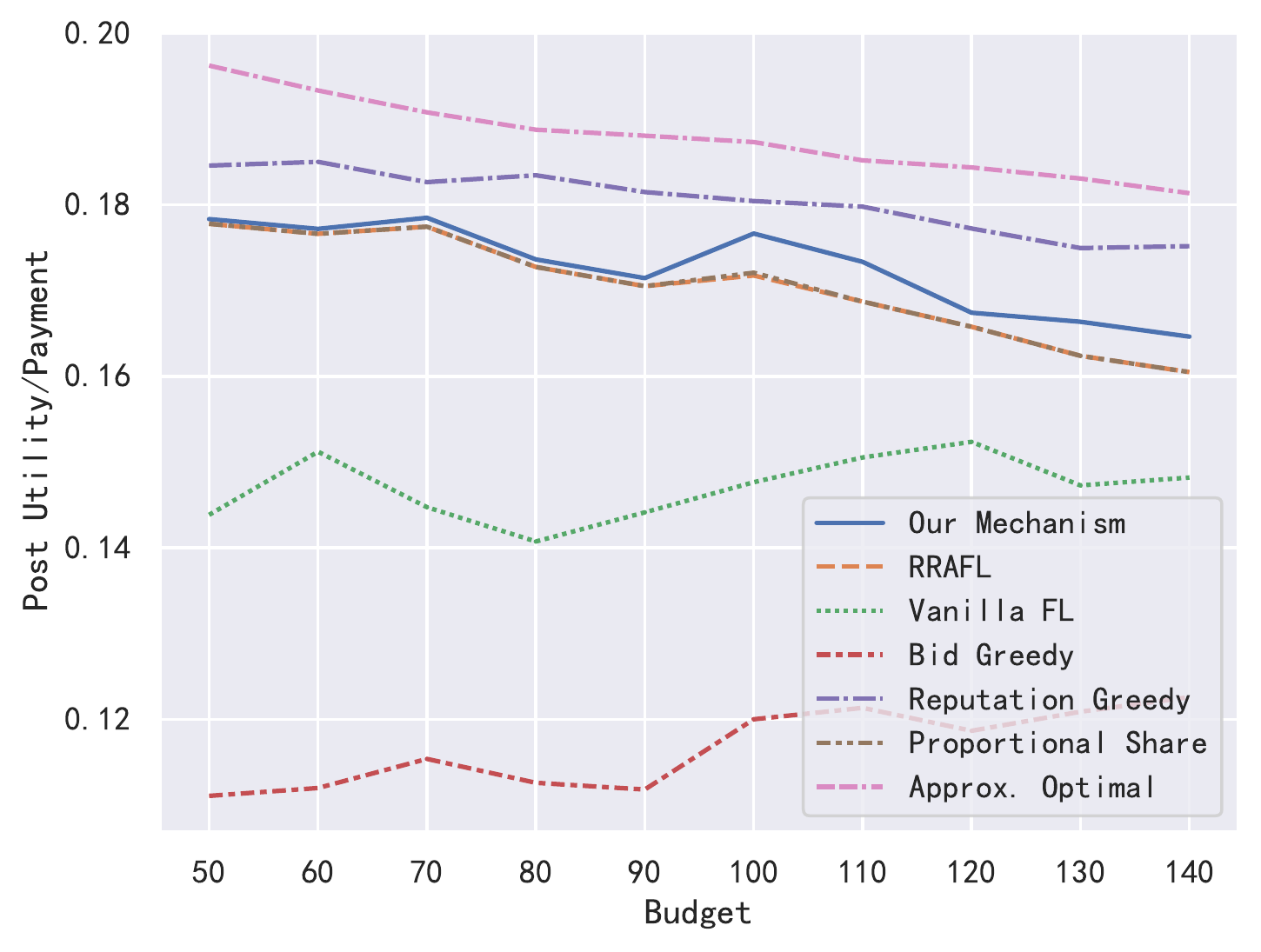}}
  \caption{Impact of budget on the ex-post utility of unit payments}
  \label{fig:budget_compare}
\end{figure}
\begin{figure}[tb]
  \centering
  \centerline{\includegraphics[width=\linewidth]{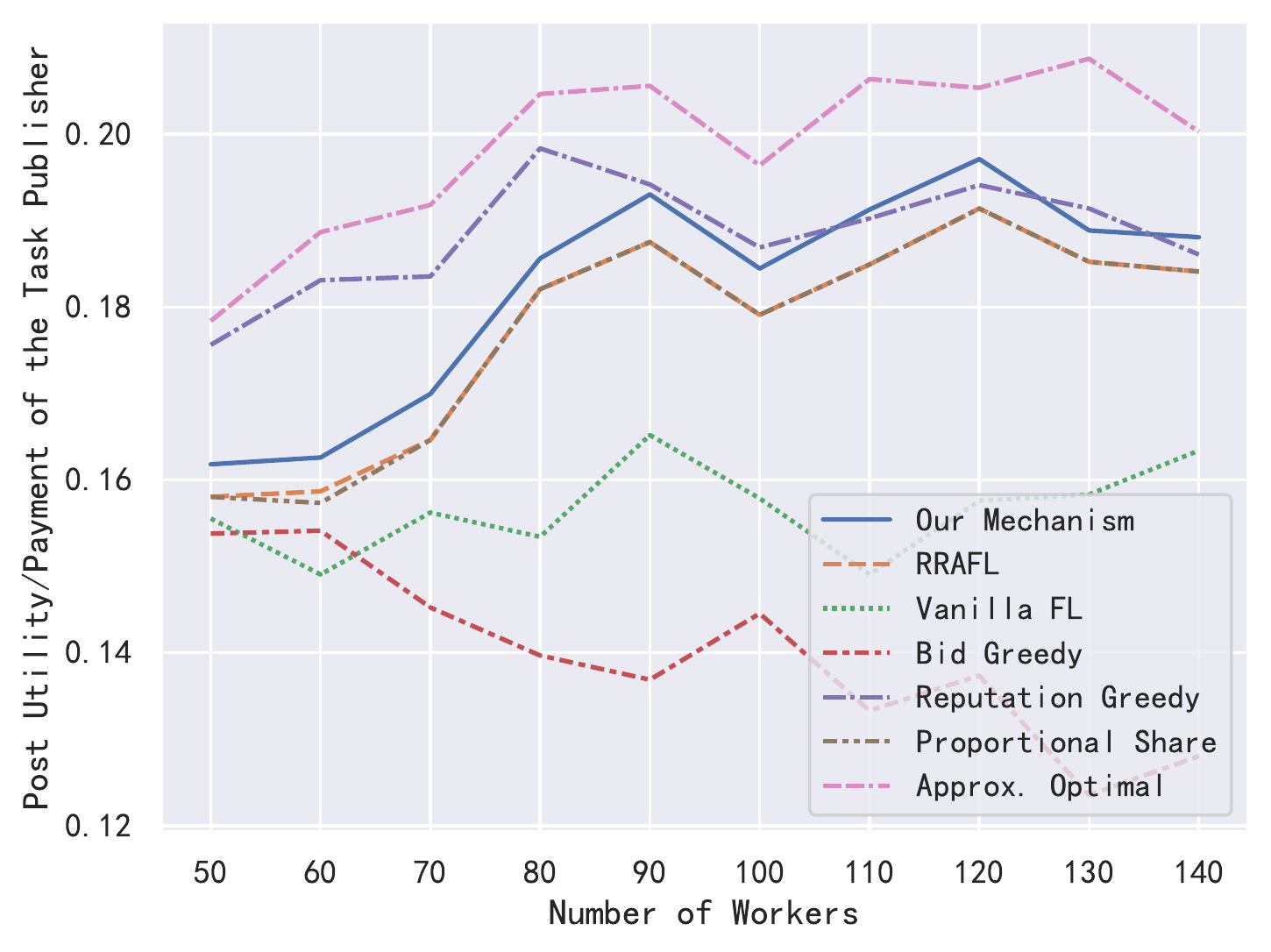}}
  \caption{Impact of the number of workers on the ex-post utility of unit payments}
  \label{fig:worker_cnt_compare}
\end{figure}

We first set up 100 workers whose $Re_i$, $b_i$ and $re_i$ are uniformly generated from $[0.1, 1]$, $[\frac{10}{3}Re_i+\frac{2}{3}, \frac{10}{3}Re_i+\frac{8}{3}]$ and $[\max(0, Re_i- 0.1), \min(1, Re_i+0.1)]$, respectively. Figure~\ref{fig:budget_compare} implies the impact on the ex-post utility of unit payments as the budget increases. When the budget is 125, Figure~\ref{fig:worker_cnt_compare} shows the effect of the number of workers. The Approx. Optimal and the Reputation Greedy mechanism well perform because full knowledge of all workers are known a priori or prefer high reputation but do not guarantee truthfulness. Our mechanism sacrifices utility to achieve truthfulness while outperforming other benchmarks. These mean that our mechanism can help to get as much utility as possible in unit payment. 

\begin{table}[tb]
  \centering
  \begin{tabular}{cccccc}
  \toprule
  \textbf{Dataset} & \textbf{Data Acc.} & \textbf{0.1} & \textbf{0.4} & \textbf{0.7} & \textbf{1.0}  \\
  \toprule
  \multirow{3}{*}{MNIST} & \textbf{Contribution} & 0.430 & 0.626 & 0.824 & 0.988  \\
  & \textbf{Reputation} & 0 & 0.421 & 0.556 & 0.985  \\
  & \textbf{Payment} & 0 & 2.203 &  2.663 & 5.533  \\
  \midrule
  \multirow{3}{*}{\shortstack{Fashion\\MNIST}} & \textbf{Contribution} & 0.734 & 0.785 & 0.867 & 0.929  \\
  & \textbf{Reputation} & 0.291 & 0.549 & 0.666 & 0.905  \\
  & \textbf{Reward} & 0.999 & 2.539 & 3.235 & 5.527  \\
  \bottomrule
  \end{tabular}
  \caption{Contribution, reputation and payment for different data acc.}
  \label{tab:contribution_reward_reputation}
\end{table}

Then we set up 30 workers whose data is iid. The data accuracies $dacc$ of 15, 5, 5, and 5 workers are 1.0, 0.7, 0.4, and 0.1, respectively. $b_i$ is randomly generated from $[\frac{10}{3}dacc+\frac{2}{3}, \frac{10}{3}dacc+\frac{8}{3}]$. Run 50 tasks with 10 global iterations for the MNIST and Fashion MNIST dataset, respectively, using the last 45 tasks to evaluate the mechanism. Table~\ref{tab:contribution_reward_reputation} shows the average contribution, reputation, and reward of workers with different data accuracies in 50 tasks. We can observe that workers with higher data accuracy can obtain higher contributions, reputations, and rewards. This demonstrates that our method can effectively measure contributions and establish appropriate reputations.

\begin{table}[tb]
  \centering
  \begin{tabular}{ccc}
  \toprule
  \textbf{Mechanism} & \textbf{MNIST} & \textbf{Fashion MNIST}  \\
  \toprule
  \textbf{Our Mechanism}    & \textbf{1.000} & \textbf{1.000}  \\
  \textbf{RRAFL}            & 1.000 &  1.000 \\
  \textbf{Vanilla FL}       & 0.4903 & 0.4879 \\
  \textbf{Bid Greedy}       &  0.2830 & 0.3134  \\
  \textbf{Reputation Greedy} & 0.9512 & 0.9602  \\
  \bottomrule
  \end{tabular}
  
  \caption{Proportion of workers with a data accuracy of 1.0 among the selected workers}
  \label{tab:proportion_acc1}
\end{table}

Table~\ref{tab:proportion_acc1} shows that in the last 45 tasks, the workers selected through ours and the RRAFL mechanism are all workers with a data accuracy of 1.0. In our mechanism and RRAFL, the proportion of workers with a data accuracy of 1.0 is higher than in the others, especially Vanilla FL and Bid Greedy. Our mechanism helps task publishers select high-quality workers.

\begin{table}[htbp]
  \centering
  \begin{tabular}{ccc}
  \toprule
  \textbf{Mechanism} & \textbf{MNIST} & \textbf{Fashion MNIST}  \\
  \toprule
  \textbf{Our Mechanism}    & \textbf{0.3747} & \textbf{0.8924}  \\
  \textbf{RRAFL}            & 0.3749 &  0.8944 \\
  \textbf{Vanilla FL}       & 0.4409 & 1.0082 \\
  \textbf{Bid Greedy}       & 0.4977 & 1.1585  \\
  \textbf{Reputation Greedy} & 0.3769 & 0.8942  \\
  \bottomrule
  \end{tabular} 
  
  \caption{Average loss of global models with different mechanisms}
  \label{tab:loss}
\end{table}

Table~\ref{tab:loss} shows that the average loss of the global model obtained by our mechanism is the smallest and much lower than that of Vanilla FL and Bid Greedy, meaning that our mechanism outperforms others and can improve the model.

\section{Conclusion}
We designed an auction-based ex-post-payment federated learning incentive mechanism with reputation and contribution measurement. First, we propose a fair contribution measurement method. Second, we establish a reputation system. Third, combining reputation and reverse auction, we design a mechanism for selecting and paying workers. Finally, theoretical analysis shows the effectiveness of our mechanism and experiments show the effectiveness of our mechanism.

\bibliographystyle{named}
\bibliography{ijcai22}

\end{document}